\documentclass[11pt]{article}
\usepackage{arxiv2017}
\usepackage{times}
\usepackage{url}
\usepackage{latexsym}
\usepackage{etoolbox}
\usepackage{times}
\usepackage{helvet}
\usepackage{courier}
\usepackage[absolute,showboxes]{textpos}
\pdfoutput=1
\eaclfinalcopy
\usepackage{amsfonts}
\usepackage{amssymb}
\usepackage{subcaption}
\usepackage{standalone}
\usepackage{color}
\usepackage{color, colortbl}
\usepackage{multirow}
\definecolor{high}{rgb}{0.7,.85,.85}
\usepackage[table]{xcolor}
\usepackage{siunitx}
\definecolor{lightgray}{gray}{0.9}
\usepackage{microtype}
\usepackage{pifont}

\usepackage{framed}
\usepackage{mathtools,xparse}
\usepackage{tikz}
\usepackage{pgfplots}
\usepackage{booktabs}
\usepackage{textcomp}
\usepackage{algorithmicx}
\usepackage{amsmath}
\usepackage{algorithm}
\usepackage{algcompatible}
\usepackage{paralist}
\usepackage{enumitem}
\usepackage{breqn}
\usepackage{amsthm}

\usepackage[skip=4pt, font=small]{caption}

\DeclareMathOperator*{\argmax}{argmax}

\DeclarePairedDelimiter{\norm}{\lVert}{\rVert}
\NewDocumentCommand{\normL}{ s O{} m }{%
  \IfBooleanTF{#1}{\norm*{#3}}{\norm[#2]{#3}}_{L_2(\Omega)}%
}

\definecolor{ourgreen}{rgb}{0.0,0.5,0.0}
\newcommand{\comment}[1]{}

\usepackage{pifont}
\newcommand{\cmark}{\ding{51}}%
\newcommand{\xmark}{\ding{55}}%

\renewcommand{\footnotesize}{\scriptsize}

\newtheorem{lemma}{Lemma}

\theoremstyle{empty}
\newtheorem{duplicate}{Lemma}
\begin{document}
\title{Domain Adaptation for Named Entity Recognition in Online Media with Word Embeddings}
\author{Vivek Kulkarni\thanks{This work was done when the author was a research intern at Yahoo. }
\and
Yashar Mehdad 
\and
Troy Chevalier \\
Yahoo Research\\
Sunnyvale, USA\\
\{vvkulkarni@cs.stonybrook.edu, ymehdad@yahoo-inc.com, troyc@yahoo-inc.com\}
}
\maketitle

\begin{abstract}
Content on the Internet is heterogeneous and arises from various domains like News, Entertainment, Finance and Technology.   
Understanding such content requires identifying named entities (persons, places and organizations) as one of the key steps.
Traditionally Named Entity Recognition (NER) systems have been built using available annotated datasets (like CoNLL, MUC) and demonstrate excellent performance. 
However, these models fail to generalize onto other domains like Sports and Finance where conventions and language use can differ significantly. 
Furthermore, several domains do not have large amounts of annotated labeled data for training robust Named Entity Recognition models.  A key step towards this challenge is to adapt models learned on domains where large amounts of annotated training data  are available to domains with scarce annotated data.

In this paper, we propose methods to effectively adapt models learned on one domain onto other domains using distributed word representations.  
First we analyze the linguistic variation present across domains to identify key linguistic insights that can boost performance across domains. We propose methods to capture domain specific semantics of word usage in addition to global semantics. 
We then demonstrate how to effectively use such domain specific knowledge to learn NER models that outperform previous baselines in the domain adaptation setting. 
\end{abstract}
\section{Introduction}
\label{sec:Introduction}
Named Entity Recognition (NER) is a critical task for understanding textual content. 
While most NER systems demonstrate very good performance, this performance is typically measured on test data drawn from the same domain as the training data.

\begin{figure}[ht!]
	\centering
	\includegraphics[trim = 1in 0.7in 1in 0.7in, clip, width=0.9\columnwidth]{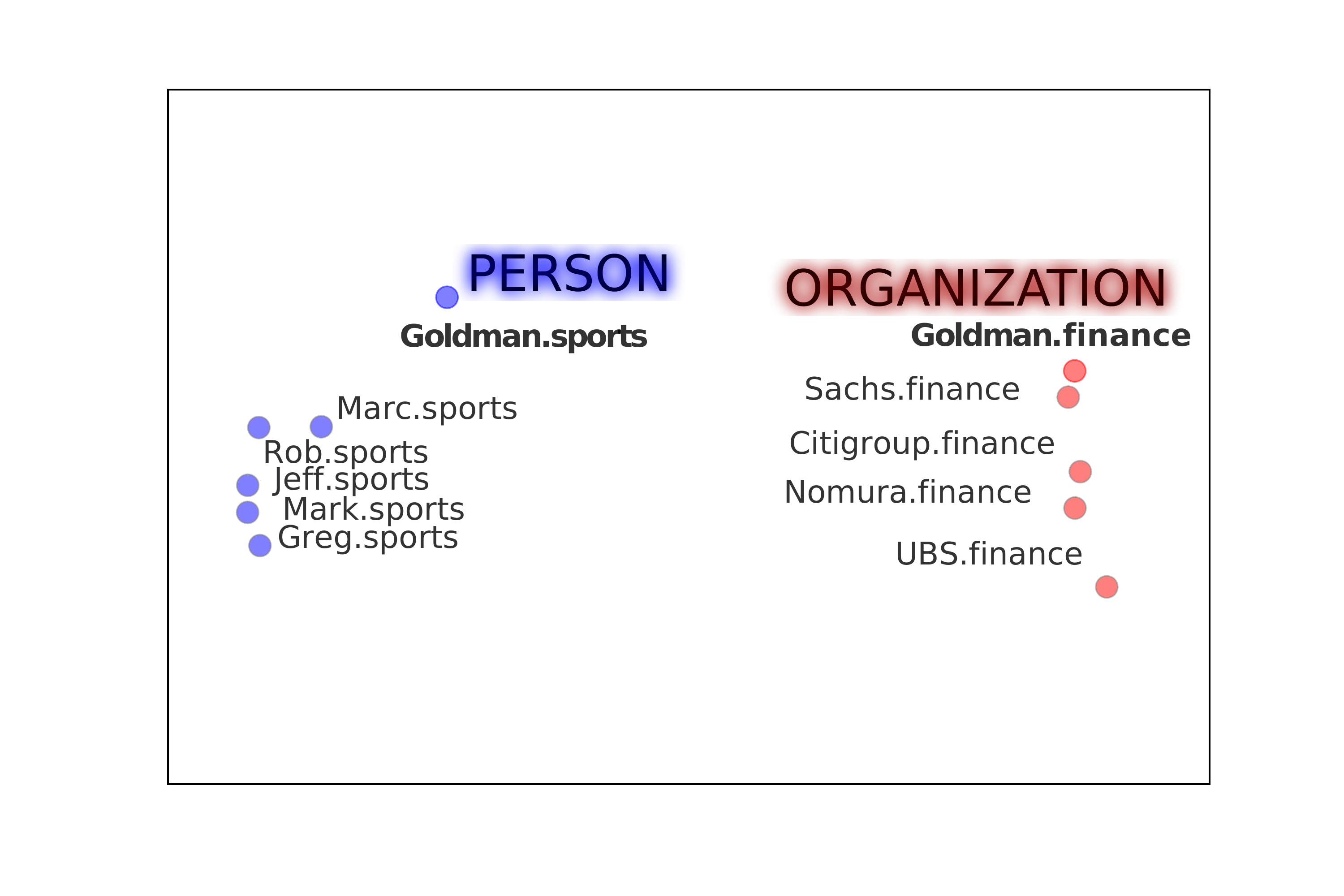}
	\vspace{-0.05in}
	\caption{A 2-D projection of the semantic space learned using \textsc{DomainDist} capturing domain specific differences in the usage of the word \texttt{Goldman} between Sports and Finance. Note how \texttt{Goldman} is close to other banks in Finance domain, but close to other person names in Sports. Capturing such domain specific differences explicitly can allow a model to more effectively infer that \texttt{Goldman} is an \emph{Organization} in Finance but a \emph{Person} in Sports.}
	\label{fig:cw}
\end{figure}
 
For example, most competitive Named Entity Recognition systems are trained on large amounts of labeled data from a given domain (like CoNLL or MUC) and evaluated on a held out test set drawn from the same domain \cite{florian2003named,chieu2002named,ando2005framework,senna,huang2015}. 
While such systems demonstrate high performance in-domain, content on the Internet can originate from multiple domains like Finance and Sports over which these systems perform quite poorly.
Moreover one typically does not have access to large amounts of labeled examples on these domains to train robust domain specific models.  
This challenge is typically addressed through domain adaptation techniques \cite{blitzer2006domain,jiang2007instance,satpal2007domain,jiang2008domain,li2012literature}. 
Most existing work on domain adaptation like Feature Sub-setting \cite{satpal2007domain}, Structural Correspondence Learning \cite{blitzer2006domain,chen2012marginalized}, learn a subset of features or learn dense representations of features that are more suited for domain adaptation. 
Different from these works,  we explore word embeddings that \emph{explicitly} capture domain specific differences while still capturing shared semantics across domains, and show that our proposed methods outperform several competitive baselines on domain adaptation for NER. 

With recent advances in representation learning, word embeddings have been shown to be very useful features for several NLP tasks like POS Tagging, NER, and Sentiment Analysis \cite{chen2013expressive,polyglot-ner,senna}.  
One drawback of using generic word embeddings is that these word vectors do not capture domain specific differences in word semantics and usage. 
To illustrate this, consider articles from two distinct domains: (a) Sports and (b) Finance.  
The word \texttt{tackle} in the Sports domain is generally associated with moves in \emph{football} and used as \emph{``A defensive tackle''}. However in the domain of Finance,  \texttt{tackle} is used to indicate problem solving as in \emph{``The company needs to tackle the rising costs immediately''}.

Explicitly modeling such domain specific differences allows us to capture linguistic variation between domains that serve as distinctive features to boost performance of a machine learning model on NLP tasks.  
In this work we propose methods to effectively model such domain specific differences of language. 
We then apply our methods to analyze domain specific differences in word semantics.  
Finally, we demonstrate the effectiveness of using domain specific word embeddings for the task of Named Entity Recognition in the domain adaptation setting.  Figure \ref{fig:cw} shows the domain specific differences captured by our method across two domains (a) Sports and (b) Finance. 
Observe how the domain specific embeddings that our method learns can easily capture the distinct usages of a word (in this case as a Person or an Organization). As we will show in Section \ref{sec:results} such distinctive representations can improve performance of Named Entity Recognition in different domains outperforming competitive baselines.

In a nut shell, our contributions are as follows:
\begin{itemize}
\item \textbf{Linguistic Variation across Domains}: Given a word $w$ how does its usage differ across different domains? We analyze variation in word usage (semantics) across different domains like Finance and Sports using distributed word representations (Section \ref{sec:domainvariation}). 
\item \textbf{NER systems for Sports and Finance}: We propose methods to effectively use such domain specific knowledge captured by word embeddings towards the task of Named Entity Recognition. In particular we show how to build state of the art NER systems for domains with scarce amount of annotated training data by adapting NER models learned primarily on domains with large amounts of annotated training data (Section \ref{sec:ner}). 
\end{itemize}

\section{Methods}
In this section we propose (a) Two methods to model domain specific word semantics in order to explicitly capture linguistic differences between domains 
and (b) Two methods that use domain specific word embeddings to learn robust Named Entity Recognition models for different domains using domain adaptation.
\label{sec:methods}
\vspace{-0.1in}
\subsection{Domain Specific  Linguistic Variation}
\label{sec:domainvariation}
\subsubsection{\textsc{DomainDist}}
Given a corpus $\mathcal{C}$ with $K$  domains and vocabulary $\mathcal{V}$, we seek to learn a domain specific word embedding $\phi_{k}:\mathcal{V} \mapsto \mathbb{R}^d$ using a neural language model where $k\in{\{1\cdots{K}\}}$.
We apply the method discussed in \cite{bamman2014distributed,kulkarni2015freshman} to learn domain specific word embeddings. 
\footnote{
In Section \ref{sec:domaindisamb} we differentiate ourselves from \cite{bamman2014distributed,kulkarni2015freshman} by outlining a probabilistic method that uses this model to disambiguate the domain given a phrase that outlines the usage of a word $w$.}  
We briefly describe this approach below as pertaining to learning domain specific embeddings.
For each word $w\in\mathcal{V}$ the model learns (1) A global embedding $\delta_{\text{MAIN}}(w)$ for the word ignoring all domain specific cues and (2) A differential embedding $\delta_{k}(w)$ that encodes deviations from the global embedding for $w$ specific to domain $k$.
The domain specific embedding $\phi_{k}(w)$ is computed as: $\phi_{k}(w)=\delta_{\text{MAIN}}(w) + \delta_{k}(w)$. 
The global word embeddings are randomly initialized, while the differential word embeddings are initialized to $\mathbf{0}$. 
We use the Skip-gram objective function with hierarchical soft-max  to learn the global and the differential embeddings. 
We set the learning rate $\alpha=0.025$, context window size $m$ to $10$ and word embedding size $d$ to be $100$. An example of the domain specific linguistic variation captured by \textsc{DomainDist} is illustrated in Figure \ref{fig:cw}.

\comment{
\begin{figure}[t!]
\vspace{-0.05in}
 \begin{center}
		\includegraphics[trim = 1in 0.7in 1in 0.7in, clip, width=\columnwidth]{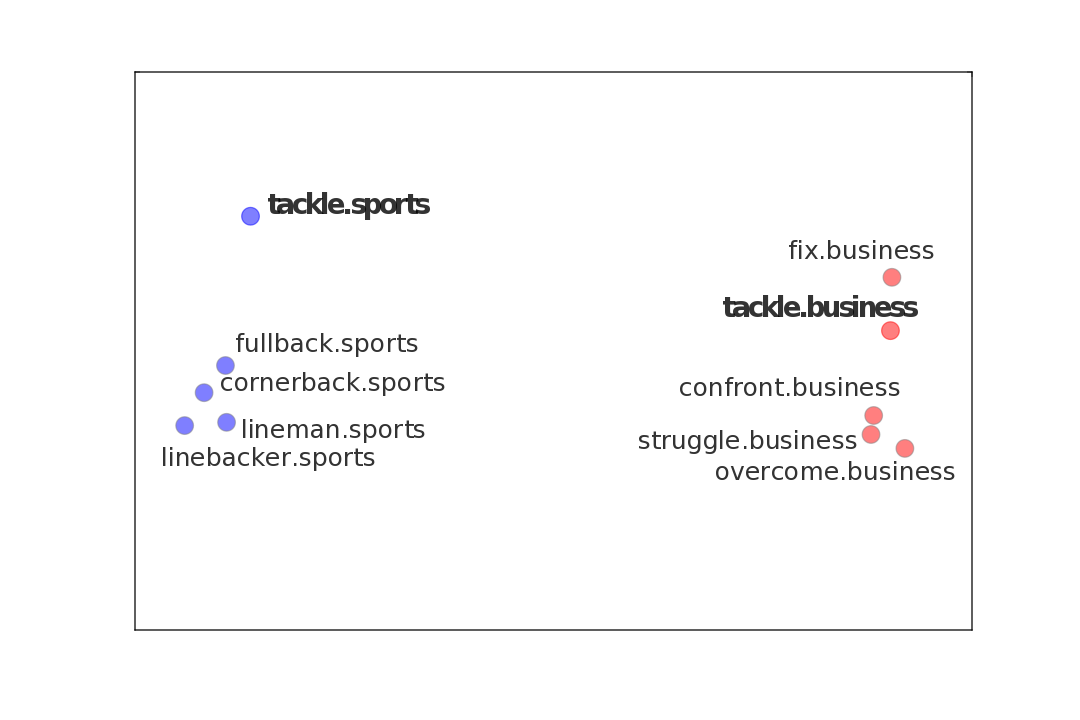}
		\caption{A 2-D projection of the semantic space capturing domain specific differences in the usage of the word \texttt{tackle} between Sports and Finance. Observe how \texttt{tackle} is close to words like \texttt{fix}, \texttt{overcome} and has the sense of \emph{problem solving} in Finance. This is in contrast to Sports where \texttt{tackle} is a  move in \emph{American football} and is close to other American football terms like \texttt{quarterback} and \texttt{lineback}.}
	\label{fig:dist}
	\end{center}
	\vspace{-0.1in}
\end{figure} 
}
\subsubsection{\textsc{DomainSense}}
Here we outline yet another method to capture semantic variation in word usage across domains. We model the problem as follows:
\begin{itemize}[noitemsep, nolistsep]
\item \textbf{Sense Specific Embeddings} We assume each word $w$ has potentially $S$ senses where we seek to learn not only an embedding for each sense of $w$ but also infer what these senses are from the corpus $\mathcal{C}$.
\item \textbf{Sense Proportions in Domains} The usage of $w$ in each domain $d$ can be characterized by a probability distribution $\pi_{d}(w)$ over the inferred senses of $w$.
\end{itemize}   

To learn sense specific embeddings, we use the Adaptive Skip-gram model proposed by \cite{bartunov2015breaking} to automatically infer (a) the different senses a word $w$ exhibits (b) a probability distribution $\pi(w)$ over the the different senses a word exhibits in the corpus and (c) an embedding for each sense of the word. Specifically, we combine the sub-corpora of different domains to form a single corpus $\mathcal{C}$.  We then learn sense specific embeddings for each word $w$ in $\mathcal{C}$ using the Adaptive Skip-gram model.
We set the number of dimensions $d$ of the embedding to $100$, the maximum number of senses a word has $S=5$ and restrict the vocabulary to only words that occur more than $100$ times. 

Finally, given a word $w$ we quantify the difference in the sense usage of $w$ between two domains $d_{i}$ and $d_{j}$ as follows:
\begin{enumerate}
\item Disambiguate each occurrence of $w$ in $d_{i}$ and $d_{j}$ using the method described by \cite{bartunov2015breaking}. We can then estimate the sense distribution of word $w$ in domain $d_{i}$, $\pi_{d_{i}}(w)$  as 
$Pr(\pi_{d_{i}}(w) = s) = \frac{\#_{d_{i}}(Sense(w) = s)}{\#_{d_{i}}(w)}$
where $\#_{d_{i}}(X)$ represents the count of number of times $X$ is true in domain $d_{i}$.
\item  We then compute the Jennsen-Shannon Divergence (JSD) between the sense distributions of the word $w$ between the two domains $d_{i}$ and $d_{j}$ to quantify the difference in sense usage of $w$ between these domains.
\end{enumerate}
Table \ref{tab:senses} shows a small sample of words along with their inferred senses using this method. Figure \ref{fig:dist} then depicts the domain specific difference in the sense usages of \texttt{goal} as computed by \textsc{DomainSense}.  

\begin{figure}[t!]
	\vspace{-0.05in}
	\begin{center}
		\includegraphics[width=\columnwidth]{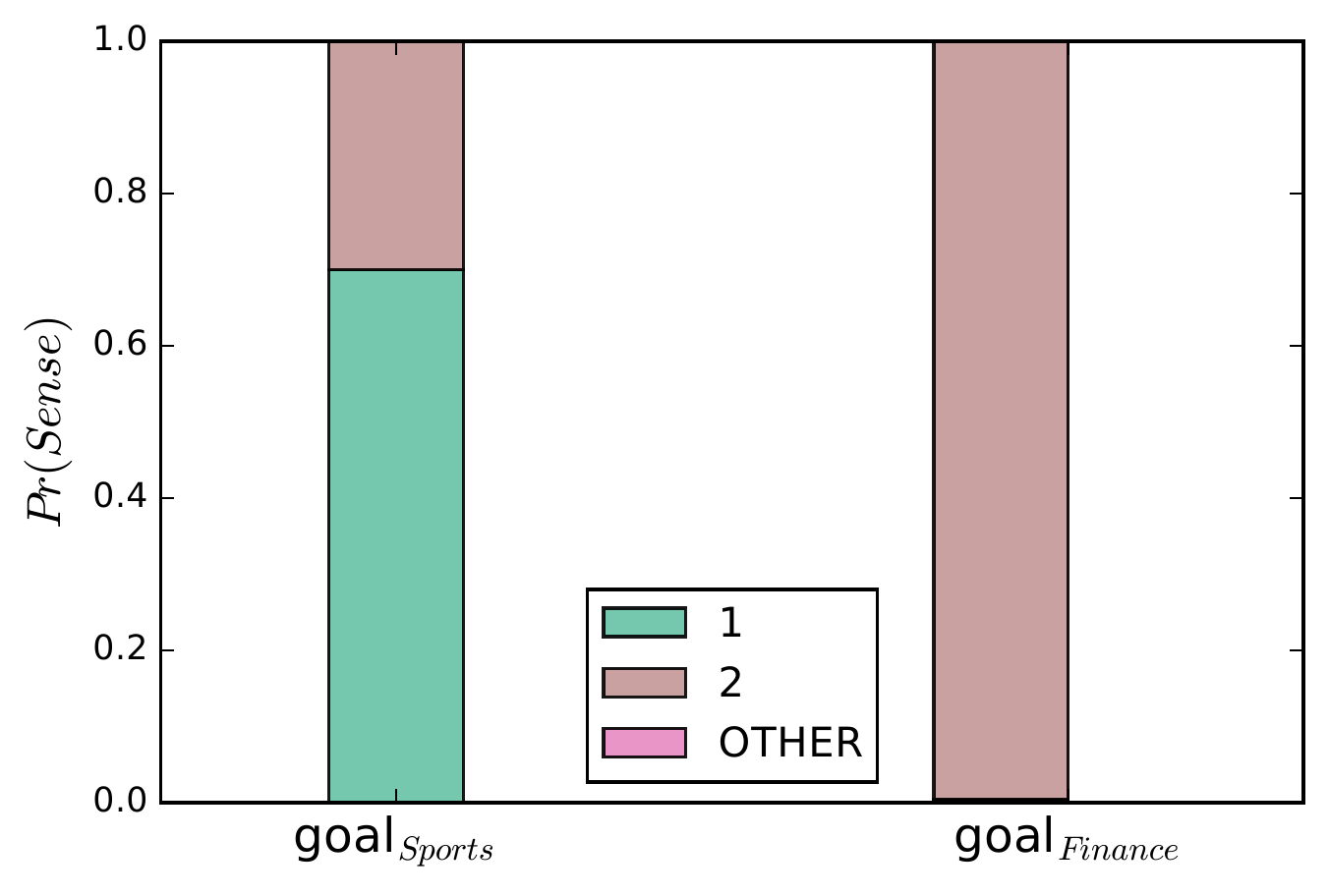}
	\end{center}
	\vspace{-0.1in}
	\caption{Different sense proportions of \texttt{goal} in Sports and Finance as computed by \textsc{DomainSense}. The word \texttt{goal} has two inferred senses as shown in Table \ref{tab:senses}: \textsc{Sense1} corresponds to the sense of \texttt{goal} as a \emph{score} in games or sports. \textsc{Sense2} corresponds to the sense of \texttt{goal} as an \emph{objective}. The usages of these senses is different in Sports and Finance. Note that in Sports, \textsc{Sense1} is dominant while in Finance the usage is exclusively \textsc{Sense2}.}
	\label{fig:dist}
\end{figure}

\begin{table*}[t!]
	\small
	\centering
	\begin{tabular}{l|c|p{4cm}|p{4cm}|p{4cm}}
		\textbf{Word} & \textbf{\#(Senses)} & \textbf{Sense 1} & \textbf{Sense 2}  & \textbf{Sense 3} \\ \hline
		tackle & 2 & handle, avoid, sidestep & linebacker, cornerback, defensive  & -- \\
		track & 3 & field, swimming, lacrosse & song, album, remix & bike, downhill, route, lane, dirt \\
		board & 2 & committee, chairperson, chair &  deck, boards, bench, boat, raft & -- \\
		heats & 2 & rounds, semifinals & cools, warmed, dried
		 & -- \\
		 goal & 2 & try, hat-trick, game-winning & aim, mission, objective & -- \\
 	\end{tabular}
 	\caption{The senses inferred for a sample set of words by Adaptive Skip-gram. Each word's sense is succinctly described by the nearest neighbors of that word's sense specific embedding. Note the different senses of words like \texttt{heats} and \texttt{tackle}. These senses are used in different proportions in various domains as shown for \texttt{goal} in Figure \ref{fig:dist}. }
	\label{tab:senses}
\end{table*}

While both \textsc{DomainDist} and \textsc{DomainSense} explicitly capture domain specific differences in word semantics, they differ in their underlying models. \textsc{DomainDist} captures domain specific word semantic/usage by directly learning domain specific word representations. \textsc{DomainSense} on the other hand infers different senses of a word and learns an embedding for each sense. Domain specific differences are then modeled by differences in sense usage of the word across domains. To illustrate this difference, consider the word \texttt{goal}. \textsc{DomainDist} will capture the fact that \texttt{goal} is associated with \texttt{match}, \texttt{winning} in Sports and capture this sense of \texttt{goal} in the \emph{Sports} Specific Embedding. \textsc{DomainSense} in contrast will infer that \texttt{goal} has two senses overall (see Table  \ref{tab:senses}) and then capture that in Sports both these senses are used. Moreover, the sense related to \texttt{score} is used $70\%$ of the time while the sense associated with \texttt{objective} is estimated to be used $30\%$ of the time in Sports. Finally, we empirically evaluate the effectiveness of both \textsc{DomainDist} and \textsc{DomainSense} for the task of NER (Section \ref{sec:results}).

\subsubsection{Error Bounds on Estimated JS Divergence}
Note that for any given word $w$, the empirical probability estimate computed as $Pr(\pi_{d_{i}}(w) = s)$ is estimated from its usage in the sample corpus and is hence a random variable. Since these estimates of probabilities are further used to compute the JS Divergence between the sense distributions over two domains $(d_{i}, d_{j})$ this estimate is also a random variable and demonstrates variance. We now provide theoretical bounds on the standard deviation of the computed Jennsen Shannon Divergence. 
\begin{lemma}
	Let $n_{a}(w)$ and $n_{b}(w)$ be  the total number of occurrences of a word $w$ in domains $d_{a}$ and $d_{b}$ respectively. The standard deviation in the JS divergence of the sense distribution of  $w$ across this domain pair is $\mathcal{O}(\frac{1}{\sqrt{n_{a}(w)}}+\frac{1}{\sqrt{n_{b}(w)}})$
	\label{lemma:js}
\end{lemma}
\begin{proof}
	We provide the proof in the supplemental material.
\end{proof}
The bound above implies the following: (a) We can quantify the uncertainty in our estimates based on the frequency of the words in the corpus and interpret our results with greater confidence. Very rare words would have larger deviations. (b) Depending on an applications sensitivity to error, we can estimate the appropriate sample complexity needed. \footnote{Our reported results (see Figure \ref{fig:adagram_table}) computing JS Divergence all have counts $>=1000$ in both domains and hence have low errors.} 
\subsection{Domain Adaptation for NER}
\label{sec:ner}
In the previous section, we described methods to capture domain specific linguistic variation in word semantics/usage by learning word embeddings that are domain specific. In this section, we outline how to learn NER models for the various domains using such word embeddings as features.  

As in previous works, we treat NER as a sequence labeling problem. To train, we use CRFsuite \cite{CRFSuite} with L-BFGS algorithm. We use a BILOU label encoding scheme. The features we use are listed in Table \ref{tab:features}. Our main features are tokens and word embeddings, within a small window of the target token.  We investigate using different kinds of embeddings listed below:
\begin{itemize}
	\item \textbf{Generic Word2vec embeddings}: We learn generic Skipgram embeddings using English Wikipedia.
	\item \textbf{Domain/Sense Specific Word Embeddings}: We experiment by using the embeddings learned using \textsc{DomainDist} and \textsc{DomainSense}.
\end{itemize}

 \begin{table*}[t!]
 \small
 	\begin{tabular}{l|p{12cm}}
 		Feature & Description \\ \hline
 		Tokens & $w_{i}$ for $i$ in $\{-2,\dots +2\}$,  $w_{i}$ and $w_{i+1}$  for $i$ in $\{-1,0\}$\\
 		Embeddings & Embeddings for $w_{i}$ for $i$ in $\{-2,\dots +2\}$ \\
 		Morphological & Shape and capitalization features, token prefixes and suffixes (up-to length 4), numbers and punctuation.
 	\end{tabular}
 	\caption{Summary of features we use for learning Named Entity Recognition (NER) models.}
 	\label{tab:features}
 \end{table*}
  
\subsubsection{\textsc{DomainEmbNER}}
Here, we outline the supervised domain adaptation method that uses domain specific word embeddings to learn NER models that significantly  outperform other baselines on NER task in the domain adaptation setting.
In this setting, we are interested in a Named Entity Recognition system for domain $\mathcal{T}$. 
However training data available for domain $\mathcal{T}$ is scarce but we have access to a  source domain $\mathcal{S}$ for which we have large number of training examples. We would like to perform domain adaptation by learning a model using the large amount of training data in source domain $\mathcal{S}$ and adapt it to work well on the target domain $\mathcal{T}$. There exist a number of methods for the task of supervised domain adaptation \cite{jiang2008domain}. 
We use a simple method for this task outlined below: 
\begin{enumerate}
	\item Combine the training data from $\mathcal{S}$ and $\mathcal{T}$. Note again that $|\mathcal{S}| >> |\mathcal{T}|$ in our setting.
	\item Extract the features outlined for training the CRF model as out-lined in Table \ref{tab:features}. Note that we experiment with different kinds of word embeddings and baselines. 
	\item Learn a CRF model using this training data.
	\item Evaluate the learned CRF model on the domain specific test data set and report the performance.
	
As we will show in Section \ref{sec:results}, using domain specific word embeddings improves the performance of NER on these target domains significantly and outperforms previous baselines for this task. 
\end{enumerate}

\subsubsection{\textsc{ActiveDomainEmbNER}}
In this section, we describe how we can learn a Named Entity Recognition system, assuming we have no labeled training data in the target domain. We can however request for a small number of examples $B$ to be labeled by annotators.  In such a setting, one can actively choose the set of examples that need to be labeled which will be most useful to learn a good model.  We propose a method to actively label examples for the purpose of domain adaptation which we describe succinctly in Algorithm \ref{alg:activelearn}. In Section \ref{sec:results} we show that by merely asking for an editorial to label $1500$ sentences, we can achieve performance close to state of art in this setting.
\begin{algorithm}[tb!]
	\small
	\caption{\small \textsc{ActiveDomainEmbNER} ($S$, $T$, $B$, $k$)}
	\label{alg:activelearn}
	\begin{algorithmic}[1]
		\REQUIRE $S$: Training data for NER in the source domain, $T$: Unlabeled data for the task of NER in the target domain which is separate and distinct from the final test set. $B$: Number of actively labeled examples, $k$: Batch size of actively labeled examples.
		\ENSURE $M$: NER model 
		\STATE $\textsc{C} \gets S$ 
		 \label{lst:line:bsbegin}
		\REPEAT
		\STATE Learn a model $M$ using $\textsc{C}$
		\STATE Evaluate $M$ on $T$.
		\STATE $E \gets $ Sort the evaluated phrases of $T$ in ascending order of model confidence (probability) and remove top $k$ least confident examples. 
		\STATE Ask an expert to label each example in $E$ and add them to $\textsc{C}$.
		\STATE $\textsc{C} \gets \textsc{C} \cup E$
		\UNTIL{$|\textsc{C}| \geq |S|+ B$} 
		\STATE \textbf{return} $M$
	\end{algorithmic}
	\vspace{-0.01in}
\end{algorithm}

\section{Datasets}
\label{sec:datasets}
In this section, we outline details of the datasets we consider for our experiments.

Our datasets can be classified into 2 categories (a) Unlabeled data for learning word embeddings and (b) Labeled data for the task of NER, each of which we describe below.

\subsection{Unlabeled Data}
We use the following unlabeled data sets for the purpose of learning word embeddings. We consider (a) all sentences of English Wikipedia (b) a random sample of 1 Million articles from Yahoo! Finance restricting our language to only English and (c) a random sample of 1 Million articles from Yahoo! Sports restricting our language to only English.

\subsection{Labeled Data}
We also use labeled data sets for the task of learning NER models which we summarize in Table \ref{tab:datasets_labeled}.


\begin{table}[t!]
\small
	\begin{tabular}{l|p{1cm}p{1cm}p{1cm}}
		& ConLL & Yahoo Finance & Yahoo Sports \\ \hline
		\# Sents (train) & $14808$  & $6439$ &
		 $4077$ \\
	   \# Sents (test) & $3648$ & $ 4294$&
	$2719$ \\		  
		Domain & News & Finance & Sports \\
	\end{tabular}
	\caption{Summary of our editorially labeled data.}
	\label{tab:datasets_labeled}
\end{table}

\section{Experiments}
\label{sec:results}
Here, we briefly describe the results of our experiments on (a) Domain Specific Linguistic Variation and (b) Domain Adaptation for Named Entity Recognition.
\subsection{Domain Specific Linguistic Variation}
Table \ref{tab:domain_dist} shows some of the semantic differences in word usage captured by \textsc{DomainDist}. Observe that the method is able to capture words like \texttt{quote, overtime, hurdles} that have alternative meanings (semantics) in a domain. For example, the word \texttt{hurdles} means \texttt{challenges} in Finance but a kind of athletic \texttt{race} in Sports.  In addition to capturing words that differ in semantics, note that \textsc{DomainDist} also uncovers differing semantic usages of entities as well, as depicted in Figure \ref{fig:cw}.  In the domain of Finance, \texttt{Anthem} refers to a health insurance company but \texttt{Anthem} in Sports dominantly refers to a \texttt{song} like a team anthem.  In Figure \ref{fig:adagram_table} a sample set of words detected by \textsc{DomainSense} are shown. Note once again, that we are able to capture domain specific differences between words (both entities and non-entities). Furthermore, \textsc{DomainSense} is able to quantify the proportion of each word sense usage in various domains. For example, the word \texttt{tackle} is used exclusively in Finance as a verb that means \texttt{to solve}, whereas in Sports \texttt{tackle} is dominantly used to refer to an \texttt{American football move}. 
Note that in Sports, the sense of \texttt{tackle} that means \texttt{to solve} is only used $~30\%$ of the time.

This ability to capture differing entity roles (like Organizations and Persons) provides an insight into the effectiveness of domain specific embeddings for improved performance on Named Entity Recognition.

\begin{table*}[t!]
	\small
	\centering
	\begin{tabular}{l|c|p{6cm}|p{6cm}}
		\textbf{Word} & \textbf{Distance} & \textbf{Usage(Finance)} & \textbf{Usage(Sports)} \\ \hline
		quote & 0.70 & an official document (used as ``details of the quote'' ) & an aphorism (a saying) \\
		selections & 0.94 & selections of menus, checkouts, products & selection to an honor (an award, recognition)\\
		overtime & 0.93 &  used as ``overtime pay'' & A checkpoint in a match (used as ``double overtime'') \\
		Assists & 0.89 & Assist, Coordinate (as in help) & A term in American football \\
		hurdles & 0.89 & setbacks, obstacles & a type of athletic race \\
		Anthem & 0.97& Health Insurance Company (similar to Aetna, Metlife) & Song (of a band, team etc) used as ``Sing the anthem'' \\
		Hays & 0.88 & Hays Advertising & Last Name of person\\
		Schneider & 0.88 & Schneider Electric (company) & Last name of a person \\
		Hugo & 0.88 & Name of a company (like Hugo Boss) & First Name of a person \\
	\end{tabular}
	\caption{Sample words that depict the differences (and the measured distance) in word semantics between Sports and Finance by \textsc{DomainDist}. Note that we capture semantic differences in words that are entities (\texttt{Anthem, Schneider}) and non-entities (\texttt{quote, overtime}).}
	\label{tab:domain_dist}
\end{table*}

\subsection{Domain Adaptation for Named Entity Recognition}
In this section, we report the results of using our \textsc{DomainDIst} and \textsc{DomainSense} word embeddings for the task of Named Entity Recognition on Finance and Sports Domains in the domain adaptation setting as described in Section \ref{sec:methods}. We also outline the baseline methods we compare to below:
\subsubsection{Baseline methods}
Since our setting is the setting of domain adaptation for Named Entity Recognition, we consider several competitive baselines for this task:
\begin{itemize}[noitemsep,nolistsep]
	\item \textbf{CoNLL-only Model}: We consider a simple baseline where we train a NER model only using CoNLL data and generic Wikipedia Embeddings without any adaptation to the target domain.
	\item \textbf{Feature Subsetting}:  This domain adaptation method tries to penalize features which demonstrate large divergence between source and target domains \cite{satpal2007domain}.  It is worth noting that this models the task of NER as a classification problem and not a structured prediction problem. \footnote{We use the implementation of feature sub-setting for Named Entity Recognition provided by \url{https://github.com/siqil/udaner}.}  
	\item \textbf{Online-FLORS}: \textsc{FLORS} learns robust representations of each word based on distributional features and counts, to boost performance across domains and treats the tagging problem as a classification problem. We also use a random sample of $100K$ unlabeled sentences from each domain which \textsc{FLORS} uses to enrich the robustness of representations learned.
	We consider a scalable version of \textsc{FLORS} \cite{yin2016online}.
	\item \textbf{FEMA}: \textsc{Fema} \cite{yang2015unsupervised} learns low dimensional embeddings of the features used in a CRF model by using a variant of the Skipgram Model \cite{mikolov2013efficient}. These features can be used to learn a model for sequence tagging. While they demonstrate their method on Part of Speech tagging, the method itself is general and can be applied to other tasks like Named Entity Recognition as well and provide a nice replacement for word embeddings as features in a CRF model. We learned $100$ dimensional \textsc{FEMA} embeddings in our experiment. \footnote{We use the open source implementation  provided at \url{https://github.com/yiyang-gt/feat2vec}}
\end{itemize}


\begin{figure*}[ht!]
	\vspace{-0.05in}
	\begin{center}
		\includegraphics[width=\linewidth]{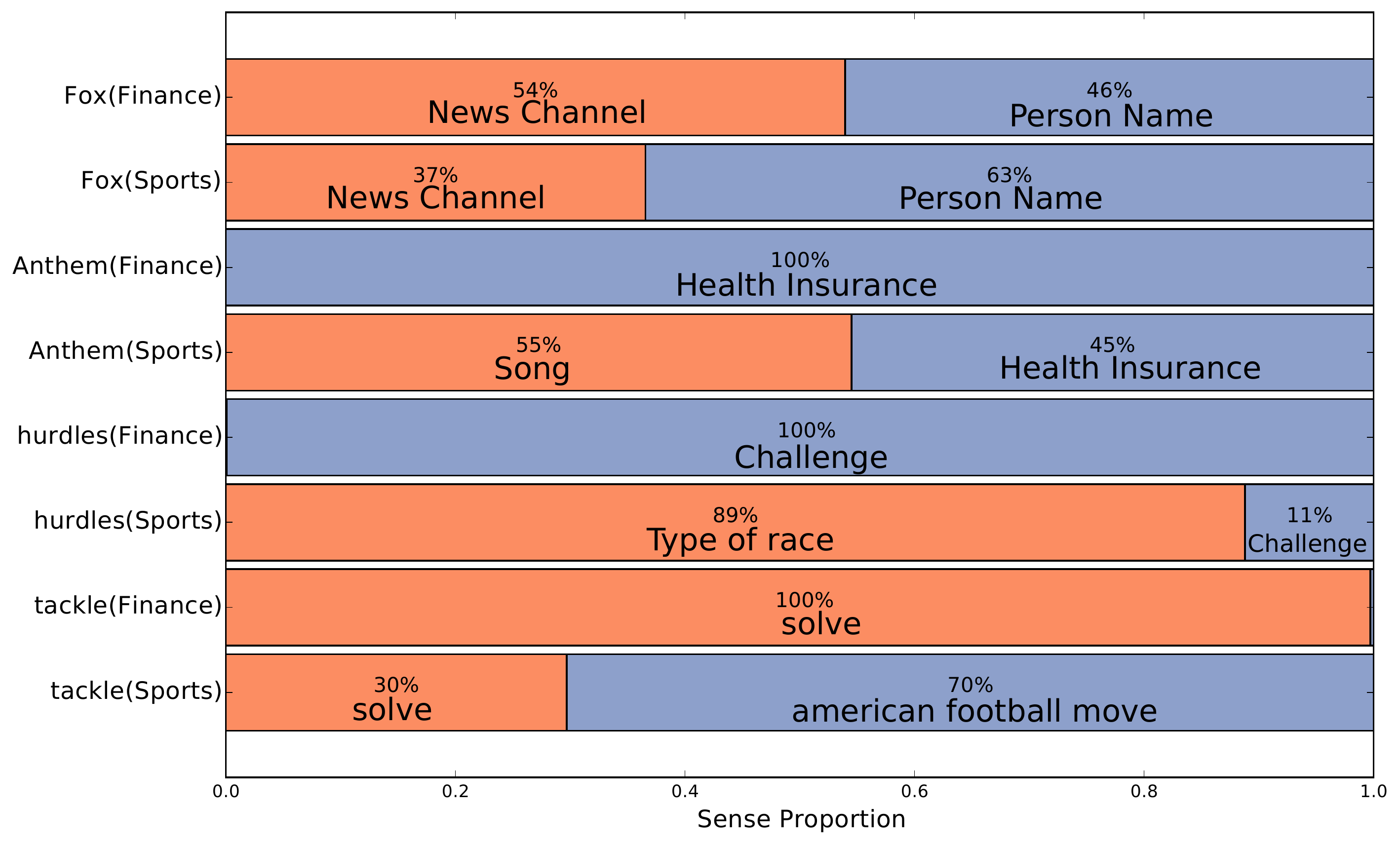}
		\caption{Sample set of words and their sense proportions in Sports and Finance as computed using \textsc{DomainSense}. Note the differences in sense usages of \texttt{Anthem}, \texttt{hurdles} and other words.}
		\label{fig:adagram_table}
	\end{center}
	\vspace{-0.1in}
\end{figure*}

\begin{table*}[tp!]
	\vspace{-0.1in}
\small
	\centering
	\begin{tabular}{l|l|ccc|ccc}
		\textbf{Data} & \textbf{Embeddings/Method} & \multicolumn{3}{c|}{Finance} & \multicolumn{3}{c}{Sports} \\
		\hline
		& & \textbf{Precision} & \textbf{Recall} & \textbf{F1} & \textbf{Precision} & \textbf{Recall} & \textbf{F1} \\ 
		ConLL-Only & Wikipedia & $51.32$ & $54.86$ & $53.03$ & $74.09$ & $68.31$ & $71.09$\\
		ConLL + Target & Feature Subsetting & $34.18$ & $45.28$ & $37.54$ &  $49.89$ & $48.63$ & $45.80$ \\
	    ConLL + Target & FLORS & $35.75$ & $46.78$ & $40.53$ & $63.48$ & $62.18$ & $62.82$\\
	    ConLL + Target & \textsc{Fema} & $67.30$ &  $68.10$ & $67.70$ & $83.18$ & $81.79$ & $82.48$ \\
	    ConLL + Target  & Wikipedia Embeddings & $70.97$ & $72.23$ & $71.59$  & $86.17$ & $85.14$ & $85.65$\\
	     ConLL + Target &  \textsc{DomainSense} embeddings & 67.22 & 68.0 & 67.61 & 83.23 &  81.82 & 82.52\\
	    ConLL + Target &  \textsc{DomainDist} embeddings & \textbf{71.62} & \textbf{72.5} & \textbf{72.06} & \textbf{85.72} & \textbf{85.88} & \textbf{85.8} \\
	\end{tabular}
	\caption{Performance of various domain adaptation methods on Named Entity Recognition in the target domains. The target domain (Target) here is one of Finance or Sports. Number of training sentences used from Finance:3219 while for Sports we use 2038 sentences. This corresponds to using only 50\% of the target domain training data to perform domain adaptation.}
	\label{tab:finance}
\end{table*}

\subsubsection{Results and Discussion}
Table \ref{tab:finance} shows the performance of our methods and other baselines on Finance and Sports. First note that a CoNLL-only model without any domain adaptation results in poor performance. Domain Adaptation methods like Feature Subsetting and \textsc{FLORS} that model Named Entity Recognition as a classification problem, rather than a sequence prediction problem perform even worse. In contrast \textsc{Fema} which learns dense representations of CRF features which can then be used to learn a more robust CRF model (that is more suited to domain adaptation) yields an significantly improved F1 score  of $67.70$ on Finance and $82.48$ on Sports respectively. Empirically we observe that using \textsc{DomainSense} embeddings improves the performance over the ConLL only model, but does not perform as well on this task (especially in Finance). We hypothesize while decomposing a word into multiple fine-grained senses is useful to capture semantic variation, using fine-grained sense embeddings for every word results in a overly complex decision space when used for tasks like NER. Finally observe that \textsc{DomainDist} which learns domain specific embeddings (without explicitly decomposing words into their senses)  outperforms all these methods in both domains. This superior performance results from the ability to capture useful broad domain specific differences more effectively.

In Table \ref{tab:proportions}, we evaluate the performance of using domain specific word embeddings \footnote{For brevity. we present results here only \textsc{DomainDist} embeddings as \textsc{DomainDist} embeddings are the best performing embeddings.} against using just generic Wikipedia based word embeddings as a function of available training data $\alpha$.  First observe that on an average, using \textsc{DomainDist} word embeddings improves the performance over using generic Wikipedia based embeddings. Observe that in general, as the amount  of training data in the \emph{target} domain increases, the advantage (gain) of using domain specific embeddings reduces. For example, when only $10\%$ of training data is available for Finance, using domain specific word embeddings results in F1 Score gain of \textbf{1.21}. However when $90\%$ of training data is available in Finance we get a small but still significant boost of \textbf{0.42} in the F1 score on using domain specific word embeddings. We explain this by noting that as the proportion of training data in the target domain increases, the model is able to pick up on domain specific cues and fine-tune its decision boundary better without needing to rely too much on the domain specific cues captured by the domain specific word embeddings.

\begin{table}[ht!]
	\small
	\centering
	\begin{tabular}{c | c | c | c | c } 
		$\alpha$ & \multicolumn{2}{c|}{\textbf{Finance}} & \multicolumn{2}{c}{\textbf{Sports}}\\
		\hline 
		& Wiki  & \textsc{DomainDist} & Wiki & \textsc{DomainDist} \\
		0.1	&  61.28 & \textbf{62.50} & 80.18 & \textbf{80.81}\\
		0.2	& 66.36 & \textbf{66.61} & 83.61 & \textbf{83.91}\\
		0.3	& 67.57 & \textbf{68.35} & 84.42 & \textbf{85.34}\\
		0.4	& 69.40 & \textbf{70.61 }& 85.49 &  \textbf{85.83}\\
		0.5	& 71.20 & \textbf{71.55} & 86.30 & \textbf{86.26}\\
		0.6	& 71.12 & \textbf{71.71} & 87.08 & \textbf{87.11}\\
		0.7	& 72.44 & 72.42 & 87.07 & \textbf{87.25}\\
		0.8	& 72.99 & \textbf{73.33} & 88.12 & \textbf{88.45}\\
		0.9	& 73.59 & \textbf{74.02} & 88.00 & \textbf{88.26}\\
	\end{tabular}
	\caption{Performance of \textsc{DomainEmbNer} using \textsc{DomainDist} Embeddings versus Wikipedia Embeddings on NER task against different proportions of training data in target domain. 
	}
	\label{tab:proportions}
\end{table}

\begin{table}[ht!]
	\small
	\centering
	\begin{tabular}{l|l|l|p{1cm}|p{1cm}}
		\textbf{\#(Sents)} & \multicolumn{2}{c|}{\textbf{Training Data \%}}
		& \textbf{Finance (F1)} & \textbf{Sports (F1)} \\
		\hline 
		& Finance & Sports & & \\
		500 & 7.7 & 12.2 & 67.68 &  83.94\\
		1000 & 15.5 & 24.5 & 69.00 &  85.53\\
		1500 & 23.2 & 36.7 & 69.79 &  86.78\\
		2000 & 31.0 & 49.0 & 70.53 & 87.14 \\
	\end{tabular}
	\caption{Performance of \textsc{ActiveDomainDistNER} on the target domains of Finance and Sports using \textsc{DomainDist} as a function of actively labeled sentences.}
	\label{tab:active_domain_dist_ner}
\end{table}

Table \ref{tab:active_domain_dist_ner} shows the performance of \textsc{ActiveDomainEmbNER} as a function of number of sentences we sought to be actively labeled. Note that merely requiring $1500$ sentences to be manually annotated, we are able to achieve close to state of art F1 performance (\textbf{69.79} on Finance and \textbf{86.78} on Sports respectively) outperforming competitive baselines.  

\section{Related Work}
Related work can be organized into two areas: (a) Socio-variational linguistics and (b) Domain Adaptation.
\vspace{-0.1in}
\label{sec:related}
\paragraph{Socio-variational linguistics}
Several works study how language varies according to geography and time \cite{eisenstein2010latent,eisenstein2011discovering,bamman2014gender,bamman2014distributed,kim-EtAl:2014:W14-25,kulkarni2015statistically,kenterad,gonccalves2014crowdsourcing,kulkarni2015freshman,cook2014novel,frermann2016bayesian,hamilton2016diachronic}. Different from these studies, our work seeks to identify semantic changes in word meaning (usage) across domains with a focus on improving performance on an NLP task like NER. The methods outlined in \cite{kulkarni2015freshman,bamman2014distributed} are most closely related to our work. While we directly build on methods outlined by them we differentiate ourselves from their work in two ways. First, we show that the model proposed in \cite{bamman2014distributed} is not only useful for learning domain specific word embeddings but the model itself can be utilized for the task of domain disambiguation by using an often neglected set of model parameters (the output vectors of the model). Second, we also present a method to quantify the semantic variation in word usage by explicitly modeling differences in the usage of different senses of words. While the methods outlined in \cite{kulkarni2015freshman,bamman2014distributed} capture domain specific differences, they do not explicitly model the fact that words have multiple senses and their usage in a domain is a mixture of different proportions over these senses which can be explicitly quantified. Finally we apply these methods to identify and analyze semantic variation in word usage across domains like Sports and Finance, highlight interesting examples of such variation prevalent across these domains with applications to Named Entity Recognition.
\vspace{-0.1in}
\paragraph{Domain Adaptation}
There is a long line of work on  domain adaptation \cite{evgeniou2004regularized,ando2005framework,blitzer2006domain,jiang2007instance,satpal2007domain,daume2009frustratingly,chen2012marginalized,schnabel2014flors,yang2015unsupervised}. Most of these works can be classified based on the strategies they use as follows: (a) Instance Weighting Methods \cite{satpal2007domain,jiang2007instance} (b) Regularization based methods \cite{evgeniou2004regularized,daume2009frustratingly} and (c) Representation Induction \cite{blitzer2006domain,chen2012marginalized,schnabel2014flors,yang2015unsupervised}. Our method of learning domain specific word embeddings in an unsupervised manner can be placed into this final category. Finally an excellent survey of various domain adaptation algorithms for NLP is provided by \cite{jiang2008domain,li2012literature}.
\section{Conclusions}
\label{sec:Conclusions}
In this paper, we proposed methods to detect and analyze semantic differences in word usage across multiple domains. 
Our methods explicitly capture domain specific cues by learning word embeddings from unlabeled text and scale well to large web scale data sets. 
Furthermore, we outline methods that leverage such domain specific linguistic variation and knowledge effectively to boost performance on NLP tasks like Named Entity Recognition on domains with scarce training data and requiring domain adaptation. 
Our methods not only out-perform previous competitive baselines but also require a very small number of manually annotated sentences in the target domain to achieve competitive performance. 
We believe our work sets the stage for new directions and further research into applications that effectively model linguistic variation across domains to improve the performance, applicability and usability of NLP systems analyzing the diverse content on the Internet.  
\section*{Acknowledgments}
{\small We thank Akshay Soni, Swayambhoo Jain, Aasish Pappu and  Kapil Thadani for valuable insights and discussions.} 
\bibliographystyle{eacl2017}
\bibliography{paper}
\appendix
\section{Supplemental Material}
\subsection{Domain Disambiguation}
\label{sec:domaindisamb}
In this section, we outline a method \textsc{DomainDist++} that builds on \textsc{DomainDist} to infer the domain given a phrase highlighting a word's usage is likely to belong to. 
Specifically, given a finite set of domains $\mathcal{D}$,  a word $w$ and a set of context words $\mathcal{T}$, we would like to infer the most likely domain $d$ which reflects the given usage of $w$.
As an illustration, suppose we have two potential domains $\mathcal{D} = \{\textsc{Sports}, \textsc{Finance}\}$ and given the usage of word \texttt{loss} as \texttt{loss of the dividend} we would like to infer that this usage is most likely from the \emph{Finance} domain.  

First note that \textsc{DomainDist} models the probability of a word $o$ given a context word $c$ as follows:

\begin{equation}
Pr(o|c, D=d) = \frac{exp(\mathbf{u_{o}}^{T}\mathbf{v_{c}})}{\sum_{\mathbf{v}\in \mathcal{V}}{exp(\mathbf{u_{o}}^{T}\mathbf{v})}}
\label{eq:skipgram}
\end{equation} 
where $D$ is discrete random variable that represents the particular domain used. The vector $\mathbf{u_{o}}$ corresponds to the output vector for $o$ while $\mathbf{v_{c}}$ corresponds to the input vector (word embedding) for word $c$ in domain $d$. 

Given a word $o$ and its context word $c$ (where $c$ is a context word in the usage of $o$), we seek to compute the probability  $Pr(D=d|o, c)$. We decompose this as follows:
\begin{dmath}
Pr(D=d|o,c) \propto Pr(o|c, D=d) \\
Pr(c|D=d) Pr(D=d) 
\label{eq:disamb}
\end{dmath}

The first term in Equation \ref{eq:disamb} is modeled by Equation \ref{eq:skipgram} which can be efficiently computed in $\mathcal{O}(\log|\mathcal{V}|)$ using hierarchical soft-max \cite{mikolov2013efficient}. The second term in Equation \ref{eq:disamb} is easily estimated by computing the relative frequency of $c$ in the corpus specific to domain $d$.  The final term is just a prior on the domains (which can be computed by relative sizes of the domain specific corpora or set to uniform).

To conclude, given a word-context pair, we can estimate the domain that characterizes this word-usage by 
\begin{equation}
\hat{d} = \argmax_{d}{Pr(D=d|o,c)}.
\end{equation}

While we discuss how to disambiguate the domain given word-context pair the above can be trivially extended when multiple context words are given. In such a case, we make an independence assumption: \emph{Each word-context pair is independent}. This decomposes the joint into a product of probabilities for each word context-pair (each of which can be computed by Equation \ref{eq:skipgram}). 

In order to evaluate our method for domain disambiguation, we consider the following three competitive approaches:

\begin{itemize}
	\item \textbf{Unigram Model} One simple method to disambiguate the domain reflecting the usage of a word $w$ along with its context words $\mathcal{T}$, is to estimate the probability of this phrase under a unigram language model specific to the domain.
	Specifically we estimate the domain as follows:
	\begin{equation}
	\hat{d} = \argmax_{d}{Pr(w|D = d)\prod_{c \in \mathcal{T}}{Pr(c|D=d)}}
	\end{equation}
	\item \textbf{DistanceMean (DM)} We consider a simple nearest neighbor based method. For each domain $d$, we compute a score, the mean cosine similarity between the word $w$ and the context words $c\in\mathcal{T}$ and choose the domain with the higher score. In summary, we estimate the domain by:
	\begin{equation}
	\hat{d} = \argmax_{d}{Score(d)}, 
	\end{equation}
	where $Score(d)$ is given by:
	\begin{equation}
	Score(d) = \frac{1}{|\mathcal{T}|}\sum_{c \in \mathcal{T}}{CosineSim(\mathbf{v_{w}}, \mathbf{v_{c}})}, 
	\end{equation}. Here $\mathbf{v_{w}}$ and $\mathbf{v_{c}}$ are the word embeddings for $w$ and $c$ specific to the considered domain $d$. 
	\item \textbf{Context Vector Mean (CVM)} We consider yet another nearest neighbor based method. 
	For each domain $d$, we first compute the mean context embedding $\mathbf{v_{\tilde{T}}}$ by averaging the domain specific word embeddings for each context word $c$. We then compute a score $Score(d)$ which is the cosine similarity between the domain specific word embedding for $w$ and the mean context vector. Therefore, we estimate the domain as follows:
	\begin{equation}
	\hat{d} = \argmax_{d}{Score(d)}
	\end{equation}
	where $Score(d) = CosineSim(\mathbf{v_{w}}, \mathbf{v_{\tilde{T}}})$ and $\mathbf{v_{\tilde{T}}} = \frac{1}{|\mathcal{T}|}\sum_{c \in \mathcal{T}}\mathbf{v_{c}}$
\end{itemize}

To highlight the differences between these methods and our method, we evaluate these methods on a small but insightful dataset of $20$ phrases which we show in Table \ref{tab:disamb_evaluation_full}. We describe our observations and conclusions briefly below:
\begin{itemize}
	\item The Unigram model fails when the context words are not distinctive of domains and occur with similar frequencies in both domains (as illustrated by the first eight examples in Table \ref{tab:disamb_evaluation_full}).
	\item \textsc{DomainDist++} is superior to  Unigram and is competitive with other nearest neighbor baselines. One drawback of the nearest neighbor methods is that they do not explicitly and interpret-ably capture strength of domain membership which \textsc{DomainDist++} captures.  To illustrate, consider disambiguating the usage of the word \texttt{on} in \texttt{we are focused on the heats}.  The word \texttt{on} does not really have distinctive semantic usages between the two domains \textsc{Sports} and \textsc{Finance}. \textsc{DomainDist++} will assign close to equal probabilities of membership of \texttt{on} to both these domains thus naturally capturing the graded membership, which is not possible using nearest neighbor methods (\textbf{DM} and \textbf{CVM}). 
\end{itemize}

\begin{table*}[ht!]
	\begin{tabular}{p{7cm}|ccccc}
		Phrase & Unigram & DM & CVM & \textsc{DomainDist++} & True Domain \\ \hline
		This is a very easy \textbf{race} and you will not need any help & \xmark & \cmark & \cmark & \cmark & Sports \\
		Other than that \textbf{race} like normal & \xmark & \cmark & \cmark & \cmark & Sports  \\
		This is a basic \textbf{race} & \xmark & \cmark & \cmark & \cmark & Sports \\
		All the strategies from the first \textbf{race} apply here & \xmark & \cmark & \cmark & \cmark & Sports \\
		strongly consider purchasing the car for this \textbf{race} & \xmark & \cmark & \cmark & \cmark & Sports \\
		The key to success in this \textbf{race} is ability & \xmark & \cmark & \cmark & \cmark  & Sports \\
		The basic skills to win this \textbf{race} require hard work & \xmark & \cmark & \cmark & \cmark & Sports  \\
		Any legal age sex \textbf{race} & \cmark & \cmark & \cmark & \cmark  & Finance\\
		we are focused on the \textbf{heats} & \xmark & \cmark & \cmark & \cmark & Sports \\
		expectations are set high for the \textbf{heats} & \xmark & \cmark & \cmark & \cmark  & Sports \\
		\textbf{serve} as powerful market players & \cmark & \xmark & \xmark & \cmark & Finance\\
		\textbf{tackle} the opponent with force & \cmark & \xmark & \xmark & \cmark & Sports\\
		\textbf{loss} cents per share & \cmark & \cmark & \cmark & \cmark & Finance\\
		\textbf{tackle} record one solo & \cmark & \cmark & \cmark & \cmark & Sports \\
		strong \textbf{track} record and low turnover & \cmark & \cmark & \cmark & \cmark & Finance \\
		business was producing \textbf{loss} & \cmark & \cmark & \cmark & \cmark & Finance\\
		disappointed but proud of the \textbf{race} & \cmark & \cmark & \cmark & \cmark & Sports \\
		Talk to him to start the \textbf{race} & \cmark & \cmark & \cmark & \cmark & Sports \\
		water heaters \textbf{heats} on demand & \cmark & \cmark & \cmark & \cmark & Finance \\
		\textbf{tackle} defensive prospect & \cmark & \cmark & \cmark & \cmark & Sports \\	     
	\end{tabular}
	\caption{Set of words (in bold) and their usages disambiguated by all the 4 methods for evaluation. \cmark\  indicates the method predicted the correct domain while \xmark\  indicates the method predicted the incorrect domain. The \emph{Unigram} model is easily mis-led by common words across both domains which are not domain distinctive.}
	\label{tab:disamb_evaluation_full}
\end{table*}
\subsection{Error bounds on computation of JS Divergence for \textsc{DomainSense}}
\begin{duplicate}[Lemma~\ref{lemma:js}]
Let $n_{a}(w)$ and $n_{b}(w)$ be  the total number of occurrences of a word $w$ in domains $d_{a}$ and $d_{b}$ respectively. The standard deviation in the JS divergence of the sense distribution of  $w$ across this domain pair is $\mathcal{O}(\frac{1}{\sqrt{n_{a}(w)}}+\frac{1}{\sqrt{n_{b}(w)}})$
\end{duplicate}

\begin{proof}
Assume that a word $w$ has $S$ senses where the probability distribution over the senses in domain $d_{a}$ is given by $\textbf{p}=(p_{1}, p_{2}, \dots p_{S})$ and that in domain $d_{b}$ is given by $\textbf{q}=(q_{1}, q_{2}, \dots q_{S})$.

The JS Divergence between the probability distributions $\textbf{p}$ and $\textbf{q}$ is given by:
\begin{equation}
\begin{split}
JS(\textbf{p}, \textbf{q}) &= \frac{\sum{p_{i}\log{p_{i}}}+ \sum{q_{i}\log{q_{i}}}}{2} - \\ 
& \sum{\frac{p_{i} + q_{i}}{2}\log{\frac{(p_{i} + q_{i})}{2}}}
\end{split}
\end{equation}
Now note that each of $p_{i}$ and $q_{i}$ are sample MLE estimates of multinomial distribution. The standard deviations of each of these sample MLE estimates denoted by $\sigma_{p_{i}}$ and $\sigma_{q_{i}}$ are as follows:
\begin{equation}
\sigma_{p_{i}} = \sqrt{\frac{p_{i}^*(1-p_{i}^*)}{n_{a}(w)}}
\end{equation} 
\begin{equation}
\sigma_{q_{i}} = \sqrt{\frac{q_{i}^*(1-q_{i}^*)}{n_{b}(w)}}
\end{equation}
\end{proof}
where $p_{i}^*, q_{i}^*$ are the true values of the particular probabilities. 
In order to quantify the standard deviation in the resulting computation of $JS(\textbf{p}, \textbf{q})$ which we denote by $\sigma_{JS(\textbf{p}, \textbf{q})}$, we now apply the rule for \emph{propagation of uncertainty}\footnote{We ignore covariance terms as each parameter is estimated independently.}, which yields:

\begin{equation}
\sigma_{JS(\mathbf{p}, \mathbf{q})} = \sqrt{
	\begin{aligned}\sum{ (\frac{\partial{JS(\mathbf{p},\mathbf{q})}}{\partial{p_{i}}})^{2} \sigma_{p_{i}^{2}}} +  \\ 	
		 \sum{(\frac{\partial{JS(\mathbf{p}, \mathbf{q})}}{\partial{q_{i}}})^{2}\sigma_{q_{i}^{2}}}
\end{aligned}
}  
\label{eq:uncertainty}
\end{equation}

The coefficients $\frac{\partial{JS(\textbf{p}, \textbf{q})}}{\partial{p_{i}}}$ and $\frac{\partial{JS(\textbf{p}, \textbf{q})}}{\partial{q_{i}}}$ are called the \emph{sensitivity coefficients}. Substituting the computations for $\sigma_{p_{i}}$ and $\sigma_{q_{i}}$ in Equation \ref{eq:uncertainty} completes the proof.  This fact that the uncertainty in the JS Divergence is inversely proportional to the  square root of the sample size enables us to get reasonably accurate estimates by choosing an appropriate sample size. 

\end{document}